\documentclass[a4paper,UKenglish,cleveref, autoref, thm-restate]{lipics-v2021}

\usepackage[acronym]{glossaries}
\usepackage{booktabs}
\usepackage{myGlossary}
\usepackage{simpleShortcuts}
\usepackage[linesnumbered]{algorithm2e}
\newcommand{\SetData}[2]{$ #1  \leftarrow #2$}

\usepackage{tikz}
\usepackage{tikz-cd}
\usetikzlibrary{graphs}
\usetikzlibrary{calc}
\usetikzlibrary{shapes}

\tikzset{
    cross/.pic = {
    \draw[rotate = 45,thick] (-#1,0) -- (#1,0);
    \draw[rotate = 45,thick] (0,-#1) -- (0, #1);
    }, pics/cross/.default=.09cm
}
\tikzset{scaled/.style={yscale=\maxy/10,xscale=\maxx/10}}
\tikzset{mySwatch/.style={inner sep=0pt, minimum size=10pt,fill,opacity=.3}}
\tikzset{myStar/.style={star,inner sep=0pt, minimum size=7pt,draw,fill=yellow!100!blue}}
\tikzset{dot/.style={radius=2pt,fill=white,fill}}
\tikzset{short/.style={-Stealth, shorten >=0.1cm,shorten <=.1cm}}
\tikzset{swatch/.style={radius=4pt}}
\def\maxx{5}    \def\maxy{5}
\def\minx{0}    \def\miny{0}
\def\max{\maxx,\maxy} \def\min{\minx, \miny}
\long\def\minMax{
  \coordinate (max) at (\max);
  \coordinate (min) at (\min);
  \coordinate (min-max) at (min|-max);
  \coordinate (max-min) at (min-|max);
}
\def\randx{\pgfmathsetmacro{\xpoint}{int(random(1,10))*\maxx/10}}
\def\randy{\pgfmathsetmacro{\ypoint}{random(1,10)*\maxy/10}\pgfmathresult}

\def\axes#1#2{
\draw[->] (min) -- (min -| max);
\draw[->] (min) -- (min |- max);
\node[below] at (max-min) {#1};
\node[left] at (min-max) {#2};
}

\def\dominator#1#2{
  \begin{pgfonlayer}{dom1}
\fill[fill=#2,opacity=.2] (#1) rectangle (max);
  \end{pgfonlayer}
}    
\def\dominating#1#2#3{
  \begin{pgfonlayer}{dom2}
\fill[fill=#2,opacity=#3] (#1) rectangle (min);
  \end{pgfonlayer}
}    

\makeatletter
\def\pgfcurrentseed{%
\pgfmathparse{\pgfmath@rnd@z}\pgfmathresult%
}
\makeatother
\pgfdeclarelayer{bg}   
\pgfdeclarelayer{dom1}   
\pgfdeclarelayer{dom2}   
\pgfsetlayers{bg,dom1,dom2,main} 
\def\dominatorOpaque#1#2{
  \begin{pgfonlayer}{bg}
  \fill[fill=#2] (#1) rectangle (max);
  \end{pgfonlayer}
}    

\let\seed\relax
\let\x\expandafter

\def\glsjocfs{
\seed
}
\newglossaryentry{fs}{
name=fs,
user1={Z},
description={feasbile set}, 
text={\ensuremath{\glsjocfs}},
symbol={$\glsjocfs$}}

\def\glsjocpf{
\seed
}
\newglossaryentry{pf}{
name=pf,
user1={P},
description={Pareto front}, 
text={\ensuremath{\glsjocpf}},
symbol={$\glsjocpf$}}

\def\glsjocapf{
\overline{\seed}
}
\newglossaryentry{apf}{
name=apf,
user1={Z},
description={arg-front of the feasible set $\seed$}, 
text={\ensuremath{\glsjocapf}},
symbol={$\glsjocapf$}}

\def\glsjocipf{
\overline{\seed}
}
\newglossaryentry{ipf}{
name=apf,
user1={Y},
description={img-front of the feasible set $\seed$}, 
text={\ensuremath{\glsjocipf}},
symbol={$\glsjocipf$}}

\def\glsjocrpf{
\widetilde{\seed}
}
\newglossaryentry{rpf}{
user1={Y},
name=rpf,
description={aproximated Pareto front of the feasible set $\seed$}, 
text={\ensuremath{\glsjocrpf}},
symbol={$\glsjocrpf$}}

\def\glsjocpd{
        \mathbin
        \ifcase\seed        
            \prec\or
            \preceq\or
            \succ\or
            \succeq\fi
}
\newglossaryentry{pd}{
user1={0},
name=pd,
description={Pareto dominance relations}, 
text={\ensuremath{\glsjocpd}},
symbol={$\glsjocpd$}}

\def\glsjocf{
\seed
}
\newglossaryentry{f}{
user1={\phi},
name=f,
description={cnf formula}, 
text={\ensuremath{\glsjocf}},
symbol={$\glsjocf$}}

\def\glsjocfa{
\widetilde{\seed}
}
\newglossaryentry{rf}{
user1={\phi},
name=fa,
description={relaxation of formula $\seed$}, 
text={\ensuremath{\glsjocfa}},
symbol={$\glsjocfa$}}

\def\glsjock{
\seed
}
\newglossaryentry{k}{
user1={K},
name=k,
description={set of k of some formula}, 
text={\ensuremath{\glsjock}},
symbol={$\glsjock$}}

\def\glsjocass{
\seed
}
\newglossaryentry{ass}{
user1={A},
name=ass,
description={assumptions of a SAT call}, 
text={\ensuremath{\glsjocass}},
symbol={$\glsjocass$}}

\def\glsjocass{
\seed
}
\newglossaryentry{ent}{
user1={\vDash},
name=ent,
description={entailment}, 
text={\ensuremath{\glsjocass}},
symbol={$\glsjocass$}}

\def\glsjocov{
\seed
}
\newglossaryentry{ov}{
user1={O},
name=ov,
description={set of order variables}, 
text={\glsjocov},
symbol={$\glsjocov$}}

\def\glsjocinc{
\seed
}
\newglossaryentry{inc}{
user1={I},
name=inc,
description={incumbent list}, 
text={\glsjocinc},
symbol={$\glsjocinc$}}

\def\<#1>{\gls{#1}}
\def\|#1|{$\gls{#1}$}

\def\ofv{F}  
\def\f{\<f>} 
\DeclareMathOperator{\ifront}{img\, front} 
\DeclareMathOperator{\afront}{front}

\def\satcall#1#2#3{                
#1#2\text{-}\text{\tt\<sat>}(#3)} 
\def\fa{\<f>[\theta]}              
\def\A{\alpha}                     
\def\k{\kappa}                     

\def\ov{o}                         
\def\ovs{O}                        

\def\mx{\boldsymbol{x}}            
\def\mxx{\boldsymbol{x}'}          

\def\m{\nu}                        
\def\o{\boldsymbol{y}}             
\def\oo{\boldsymbol{y}'}           

\def\pf{\<pf>}                     
\def\apf{\<apf>}                   
\def\ipf{\<ipf>}                   
 
\def\moco++{$\text{\<moco>}$} 
 
\def\inc{\<inc>}                   

\def\mocoInput{\KwIn{$ \pbrac{\f, V, \ofv}$ \tcp*[f]{MOCO instance}}}
\def\mocoOutput{\KwOut{$\afront_{\f}\ofv$\tcp*[f]{one arg-front}}}

\def\US*{{\tt Core-Guided}}
\def\USS*{{\tt Core-Guided-Strat}}
\def\UST*{{\tt UnsatSatMSU3}}
\def\PM*{{\tt PMinimal}}
\def\HS*{{\tt Hitting-Sets}}
\def\PMCS*{{\tt ParetoMCS}}
\def\GIA*{{\tt GIA}}
\def\OIA*{{\tt OIA}}

\newcommand{\CoreGuidedAlg}{\texttt{Core\--Guided}\xspace}
\newcommand{\CoreGuidedAlgStrat}{\texttt{Core\--Guided\--Strat}\xspace}
\newcommand{\HittingSetsAlg}{\texttt{Hitting\--Sets}\xspace}
\newcommand{\PMinimalAlg}{\texttt{P-Minimal}\xspace}
\newcommand{\ParetoMCSAlg}{\texttt{ParetoMCS}\xspace}

\def\ent{\<ent>}      
\def\false{\bot}
\def\true{\top}

\def\PBool{\brac{0, 1}}

\def\rhs*{rhs}
\def\lhs*{lhs}
\def\ie,{i.e.}
\def\eg*{e.g.}

\def\ofv{F}               
\def\ds{X}
\def\os{Y}
\SetKwData{Opt}{Opt}
\SetKwData{sato}{\<sat>}
\def\ov{o}                
\def\ovs{O}               

\def\fss{\<fs>[T]}                 
\def\fs{\<fs>}            
\def\rfs{\<fs>[T]}        
\def\f{\<f>}              
\def\ff{\<f>[\psi]}       

\def\rf{\psi}             
\def\rff{\psi'}           
\def\rffi{\<f>[\psi_{1}]} 
\def\rffx{\<f>[\psi_{i}]} 
\def\rffn{\<f>[\psi_{n}]} 

\def\apf{\<apf>[\fs]}     
\def\rapf{\<apf>[\rfs]}   
\def\rapffi{\<apf>[T_1]}  
\def\rapffx{\<apf>[T_i]}  
\def\rapffn{\<apf>[T_n]}  

\def\ipf{\<ipf>}          

\xdef\diagno{\Delta}      

\newacronym{moco}{MOCO}{Multi-Objective Combinatorial Optimization}
\newacronym{mcs}{MCS}{Minimal Correction subset}
\newacronym{mspv}{MSPV}{Minimal Set of Positive Variables}
\newacronym{cnf}{CNF}{Conjunctive Normal Form}
\newacronym{pb}{PB}{Pseudo-Boolean}
\newacronym{pbo}{PBO}{Pseudo-Boolean Optimization}
\newacronym{csp}{CSP}{Constraint Satisfaction Problem}
\newacronym{ipm}{IPM}{Infeasible Partial Model}
\newacronym{sat}{SAT}{Satisfiability}
\newacronym{vmc}{VMC}{Virtual Machine Consolidation}
\newacronym{spl}{SPL}{Software Product Line}
\newacronym{ftp}{FTP}{Flying Tourist Problem}
\newacronym{sc}{SC}{Set Covering}
\newacronym{dal}{DAL}{Development Assurance Levels}
\newacronym{msat}{MaxSAT}{Maximum Satisfiability}
\newacronym{oia}{\OIA*}{Opportunistic Improvement Algorithm}
\newacronym{gia}{\GIA*}{Guided Improvement Algorithm}
\newacronym{sd}{SD}{Selection Delimiter}
\newacronym{swc}{SWC}{Sequential Weighted Counter}
\newacronym{gte}{GTE}{Generalized Total Encoder}
\newacronym{lbs}{LBS}{Lower Bound Set}

\graphicspath{{./}{./images/}}

\def\into{\rightarrow}

\title{New Core-Guided and Hitting Set Algorithms for Multi-Objective Combinatorial Optimization} 

\titlerunning{New Core-Guided and Hitting Set Algorithms for MOCO} 

\author{Jo\~ao Cortes}{INESC-ID - Instituto Superior Técnico, Universidade de Lisboa, Portugal}{joao.cortes@tecnico.ulisboa.pt}{https://orcid.org/0000-0003-4833-8054}{}
\author{In\^es Lynce}{INESC-ID - Instituto Superior Técnico, Universidade de Lisboa, Portugal}{ines.lynce@tecnico.ulisboa.pt}{https://orcid.org/0000-0003-4868-415X}{}
\author{Vasco Manquinho}{INESC-ID - Instituto Superior Técnico, Universidade de Lisboa, Portugal}{vasco.manquinho@tecnico.ulisboa.pt}{https://orcid.org/0000-0002-4205-2189}{}

\authorrunning{J. Cortes and I. Lynce and V. Manquinho} 

\ccsdesc[500]{Computing methodologies~Optimization algorithms}

\keywords{Multi-Objective Combinatorial Optimization, Unsatisfiability Algorithms, Hitting Sets} 

\category{} 

\relatedversion{} 

\acknowledgements{This work is supported by Portuguese national funds through Funda\c{c}\~ao para a Ci\^encia e Tecnologia, under projects PTDC/CCI-COM/31198/2017, DSAIPA/AI/0044/2018, and UIDB/50021/2020. }

\nolinenumbers 

\EventNoEds{2}
\EventAcronym{CVIT}
\EventYear{2016}
\EventDate{December 24--27, 2016}
\EventLocation{Little Whinging, United Kingdom}
\EventLogo{}
\SeriesVolume{42}
\ArticleNo{23}

\begin{document}

\maketitle

\begin{abstract}
In the last decade, a plethora of algorithms for single-objective
Boolean optimization has been proposed that rely on the iterative
usage of a highly effective Propositional Satisfiability (SAT)
solver. But the use of SAT solvers in Multi-Objective Combinatorial
Optimization (MOCO) algorithms is still scarce. Due to this shortage
of efficient tools for MOCO, many real-world applications formulated
as multi-objective are simplified to single-objective, using either a
linear combination or a lexicographic ordering of the objective
functions to optimize.

In this paper, we extend the state of the art of MOCO solvers with two
novel unsatisfiability-based algorithms. The first is a core-guided
MOCO solver. The second is a hitting set-based MOCO
solver. Experimental results obtained in a wide range of benchmark
instances show that our new unsatisfiability-based algorithms can
outperform state-of-the-art SAT-based algorithms for MOCO.

\end{abstract}

\section{Introduction}
\label{sec:intro}
It is ubiquitous in real-world problems to try to optimize several
objectives simultaneously. For instance, when making a vacation plan
with multiple destinations, one wants to minimize the time spent in
airports as well as the total amount spent on the plane tickets.
However, in situations with multiple objectives, one can rarely obtain
a solution that minimizes all objective functions. It is usually the
case that decreasing the value of an objective function results in
increasing the value of another objective function.  This occurs in
many application
domains~\cite{DBLP:journals/fgcs/LiZLY20,DBLP:journals/eswa/MarquesRR19,DBLP:journals/tse/YuanB20}.

One way to deal with multi-objective problems is to transform them into
single-objective. For example, this can be achieved by defining a linear
combination of the objective functions. However, the weight of each
objective function is hard to determine. Another option is to define
a lexicographic order of the functions~\cite{BolLexiMarques-Silva2011},
but this might result in an unbalanced solution where the first function is
minimized while the remaining ones have a high value.

Another approach is to determine the Pareto front of the multi-objective
problem. In this case, we are interested in finding all Pareto-optimal
solutions, \ie, all solutions for which one cannot decrease the value of
any of the objective functions without increasing the value of another.
After determining the Pareto front, one can select a representative
subset of solutions to present the user~\cite{DBLP:journals/cor/GuerreiroMF21}.

Several frameworks based on stochastic search have been developed to 
approximate the Pareto front of Multi-Objective Combinatorial Optimization (MOCO)
problems~\cite{DBLP:conf/ppsn/DebAPM00,DBLP:journals/tec/ZhangL07}. 
There are also several exact algorithms based on iterative calls to a
satisfiability checker, such as the Opportunistic Improvement
Algorithm~\cite{DBLP:conf/ecai/Gavanelli02}.
Additionally, the Guided-Improvement Algorithm (GIA)~\cite{jackson2009guided}, is
implemented in the optimization engine of Satisfiability Modulo Theories (SMT)
solver \texttt{Z3} for finding Pareto optimal solutions of SMT formulas.
More recently, new algorithms based on the enumeration of Minimal Correction
Subsets (MCSs)~\cite{DBLP:conf/sat/Terra-NevesLM17} or 
$P$-minimal models~\cite{DBLP:conf/cp/SohBTB17} have been proposed.
A common thread to these iterative and enumeration algorithms is that
they follow a SAT-UNSAT approach.

In this paper, we propose two new UNSAT-SAT algorithms for MOCO.  In
the first algorithm, an unsatisfiable core-guided approach is used
that relies on encodings of the objective functions to cut effectively
the search space in each SAT call.  Additionally, we also propose a
hitting set based approach for MOCO where the previous core-guided
algorithm is used to enumerate a multi-objective hitting set.
Experimental results show that the proposed core-guided approach is
complementary to the existing SAT-based algorithms for MOCO, thus
extending the state of the art tools for MOCO based on SAT technology.

The paper is organized as follows. Section~\ref{sec:prelim} defines
the Multi-Objective Combinatorial Optimization problem, as well as common
notation used in the remainder of the paper.
Next, sections~\ref{sec:unsat} and~\ref{sec:hitting-set} describe the
new core-guided and hitting set-based algorithms for MOCO, as well as
proofs of correction. Experimental results and comparison with other
SAT-based algorithms are provided in section~\ref{sec:results}.
Finally, conclusions are presented in section~\ref{sec:conc}.

\section{Preliminaries}
\label{sec:prelim}
We will start with the definitions that fall in the SAT
domain. Then, we introduce the definitions specific to the problem at
hand, namely, solving MOCO problems.

\begin{definition}[Boolean \glsentrylong{sat} problem (\glsentrytext{sat})]\label{def:cnf}
  \label{def:sat}
  Consider a set of Boolean variables $V=\{x_1, \ldots, x_n\}.$ A
  \emph{literal} is either a variable $x_i \in V$ or its negation
  $\lnot x_i \equiv \bar x_i$. A \emph{clause} is a set of
  \emph{literals}, and a \emph{unary clause} contains just one
  literal. \emph{A \<cnf> formula} $\f$ is a set of clauses. A
  \emph{model} $\m$ is a set of literals, such that if $x_i \in \m$,
  then $\bar x_i \not \in \m$ and vice versa.

  The \emph{truth value} of $\f$, denoted by $\m(\f)$, is a function
  of $\m$, and is defined recursively by the following rules. First,
  the truth value of all literals is covered by
  \begin{gather*}
    \label{eq:4}
    \m(x_i) = \true, \text{ if }x_i \in \m, \\
    \m(x_i) = \false, \text{ if }\bar x_i \in \m, \\
    \m(\lnot x_i) = \lnot \m(x_i).
  \end{gather*}
  (We say $\m$ \emph{assigns} the value $\m(x_i)$ to the variable
  $x_i$ and $\lnot \m(x_i)$ to $\bar x_i$.) 
  
  Secondly, a clause $c$ is true iff it contains at least one literal
  assigned to true. Finally, formula $\f$ is true iff it contains only
  true clauses,
  \begin{equation}
    \label{eq:70}
    \m(\f) \equiv \bigwedge_{c \in \f} \m(c),\quad
    \m(c) \equiv \bigvee_{l \in c} \m(l).
  \end{equation}
  The model $\m$ \emph{satisfies} the formula $\f$ iff $\m(\f)$ is
  true. In that case, $\m$ is \emph{($\f$-)feasible}. A set of
  models is feasible iff all its elements are feasible. 

  The \emph{Boolean Satisfiability problem}, known as \<sat>
  problem, reads as follows: Given a \<cnf> formula $\f$, decide if
  there is any model $\m$ that satisfies it. In other words,
    decide if $ \exists \m: \m(\f)$.
%
  In that case $\f$ is a \emph{satisfiable} formula. Otherwise, it
  is \emph{unsatisfiable}.
\end{definition}

Our algorithms require a \<sat> solver, to be used as an Oracle. As
they run, they place queries to the Oracle, and act accordingly to its
replies. Note that a {\em true} \<sat> Oracle only answers {\em yes} or
{\em no}.
Our oracle, if the problem is unsatisfiable, replies with an explanation
of why it is so, called a \emph{core} (Definition~\ref{def:unsatCore}).
The following interface suffices for our intended use.

\begin{definition}[\<sat> solver]
  \label{def:sat-solver}
  Let $\f, \A$ be \<cnf> formulas. We call $\f$ the \emph{main
    formula} and $\A$ the \emph{assumptions}. A \<sat> solver solves
  the \<cnf>\footnote{ We may use a \<pb>
    formula(Definition~\ref{def:pb}). In that case, we assume the
    solver first translates it into \<cnf>.} instance of the
  \emph{working formula} $\omega = \f \cup \A$, \ie, decides on the
  satisfiability of $\omega$.

  A query to the solver is denoted by $\satcall{\f}{}{\A}$.
  The value returned is a pair $(\m, \k)$, containing a feasible
  model $\m$ and a \emph{core of assumptions} $\k$, \ie, a subset of
  the assumptions $\A$ contained in some \emph{core} of $\omega$.
  If the working formula $\omega$ is not satisfiable, $\m$ does not
  exist, and the call returns $(\emptyset, \bullet)$. If $\omega$ is
  satisfiable, the call returns $(\bullet, \emptyset)$.
\end{definition}

\begin{definition}[core $\k$]\label{def:unsatCore}
  Given a \<cnf> formula $\f$, we say a formula $\k$ is an
  \emph{unsatisfiable~core}~of~$\f$ iff
  \begin{align*}
    \k \subseteq \f, \\
    \k \<ent> \false.
  \end{align*}

\end{definition}

As long as a core $\k$ of formula $\f$ is not broken, by removing some
of its clauses from $\f$, the formula will remain unsatisfiable. 
Breaking it will not necessarily make the formula satisfiable, though.

\begin{definition}[\glsentrylong{mcs} (\glsentrytext{mcs})]\label{def:correction-subset}
  Let $\mu \subseteq \f$, for some \emph{unsatisfiable} \<cnf> formula
  $\f$. If $\f \setminus \mu$ is satisfiable, then $\mu$ is called a
  \emph{correction subset} of $\f$. If there is no other correction
  subset $\mu' \subset \mu$, then $\mu$ is a \emph{\<mcs>}.
\end{definition}

{\em Correcting} a formula by dropping some clauses is a particular case
of a \emph{relaxation}. If the dropped set is a correction subset,
then the obtained formula is necessarily satisfiable.

\begin{definition}[relaxing/tightening a formula]\label{def:rcnf}
  Given $\f$, we call a formula $\rf$ a \emph{relaxation} of
  $\f$ iff $\f \ent \rf$.
  We also say $\rf$ \emph{relaxes} $\f$. Conversely, $\f$
  \emph{tightens} $\rf$.
\end{definition}

Now we will review the definitions related directly to the problem we
want to solve. It is called \emph{\<moco>} (Definition~\ref{def:moco}), and
it is a generalization of \emph{\<pbo>} (Definition~\ref{def:pbo}). The
\emph{objective functions} are \<pb> (Definition~\ref{def:pb}), and so are
the \emph{clauses}. Note that the \<pb> clauses generalize the clauses
of propositional logic.

\begin{definition}[\glsentrylong{pb} function, clause, formula
  (\glsentrytext{pb})] \label{def:pb} To any linear function
  $\brac{0, 1}^n \into \Nb$, given by 
  \begin{equation}
    g(\mx) = g(x_1\ldots x_n) = \sum_i w_i x_i \quad w_i \in \Nb,
  \end{equation}
  we call an \emph{(integer linear) \<pb> function}. 
  Expressions like
  \begin{equation}
    g(\mx) \bowtie k, \quad \bowtie \;\in \lbrace \le, \ge, = \rbrace,
  \end{equation}
  are called \emph{\<pb> clauses}. A \emph{\<pb> formula} is a set of
  \<pb> clauses. Let $\mx$ be the Boolean tuple
  $\m(V) \equiv (\m(x_1),\ldots,\m(x_n))$.  A model $\m$ satisfies a
  clause $c$ if $c(\mx)$ is $\true$. Given a formula $\f$, a model
  $\m$ is said \emph{($\f$-)feasible} if it satisfies every clause in
  $\f$. If a feasible model $\m$ exists, then $\f$ is
  \emph{satisfiable}, and $\m$ \emph{satisfies} $\f$.  The set of
  Boolean tuples
  $\fs = \brac{\mx \in \brac{0, 1}^n: \exists \m: \f(\m), \mx =
    \m(V)}$ is called \emph{feasible space} of the formula $\f$, and
  its elements $\mx$ are called \emph{feasible points}. Any subset of
  the feasible space is called a \emph{$\f$-feasible set}.
\end{definition}
Note that there is a one-to-one relation between the feasible points
of a formula, as defined above, and the set of feasible models. It is
given by $\mx = \m(V)$. This redundancy is meant to help combine the
notations.

\begin{definition}[\glsentrylong{pbo}, \glsentrytext{pbo}]
\label{def:pbo}
Let $\f$ be a \<pb> formula, and $f$ be a \<pb> function. Then,
minimize the value of the objective over the feasible space $\fs$ of the
formula. That is,
  \begin{equation}
    \text{find } \argmin_{\mx \in \fs(\f)} f.
  \end{equation}
\end{definition}

We will generalize this problem to the multi-objective case, but
first let us introduce what we mean by optimizing several objectives at
once.

Multi-objective optimization builds upon a criterion of comparison (or
order) of tuples of numbers. The most celebrated one is called
\emph{Pareto order or dominance} (Definition~\ref{def:pareto_order}).

\begin{definition}[Pareto partial order ($\gls{pd}$)]
  \label{def:pareto_order}
  Let $\os$ be some subset of $\Nb^n$.
  For any $\o, \oo \in \os$,
  \begin{gather*}
    \o\<pd>[1] \oo \iff \forall i, \o_i \leq \oo_i, \\
    \o\<pd>[0] \oo \iff     \o\<pd>[1] \oo \land \o \neq \oo, \\
    \o\<pd>[2] \oo \iff \oo \<pd>[0] \o, \\
    \o\<pd>[3] \oo \iff \oo \<pd>[1] \o.
  \end{gather*}

  We say $\o$ \emph{dominates} $\oo$ iff $\o \<pd>[1] \oo$. We say
  $\o$ \emph{strictly-dominates} $\oo$ iff $\o \<pd>[0] \oo$.
\end{definition}

Given a tuple of objective functions sharing a common domain $\ds$, we
can compare two elements $\mx,\mxx \in \ds$ by comparing the
corresponding tuples in the objective space.

\begin{definition}[Pareto Dominance ($\gls{pd}$)]
  Let $\ofv : \ds \rightarrow \os \subseteq \Nb^n$ be a
  \emph{multi-objective function}, mapping the \emph{decision space
    $\ds$} into the \emph{objective space $\os$}. For any
  $\mx, \mxx \in \ds$,
  \begin{gather*}
    \mx\<pd>[0] \mxx  \iff \ofv(\mx) \<pd>[0] \ofv(\mxx), \\
    \mx\<pd>[1] \mxx \iff \ofv(\mx) \<pd>[1] \ofv(\mxx), \\
    \mx\<pd>[2] \mxx \iff \mxx \<pd>[0] \mx, \\
    \mx\<pd>[3] \mxx \iff \mxx \<pd>[1] \mx.
  \end{gather*}
  We say $\mx$ \emph{dominates} $\mxx$ iff $\mx \<pd>[1] \mxx$. We say
  $\mx$ \emph{strictly-dominates} $\mxx$ iff $\mx \<pd>[0] \mxx$.
\end{definition}

One consequence of this choice of comparison criterion is that most
such optimizations have many different {\em good} solutions mapped to
different points in the objective space, contrary to what happens in
the single-objective case. Therefore, the solution of the problem is
actually a set, traditionally called \emph{Pareto front}. Its elements
are good in the sense that for each one there is no other that can, in
conscience, vouch for its removal.

\begin{definition}[Fronts]\label{def:pareto-front}
  Given a a multi-objective function $\ofv:\ds\rightarrow \os$ and a
  \emph{feasible space} $\fs \subseteq \ds$, the \emph{Pareto front} of $\fs$
  is a subset $\pf \subseteq \fs$ containing all elements that are not
  strictly-dominated,
  \begin{equation*}
    \pf=\brac{\mx \in \fs: \not\exists \mxx : \mxx \<pd> \mx}.
  \end{equation*}
  We call \emph{img-front} to the subset $\ipf \subseteq \os$ which is the image
  of $\pf$ by $\ofv$,
  \begin{equation*}
    \ipf \equiv \ifront_{\fs} \ofv = \brac{\o \in \os: \o = \ofv(\mx), \mx \in P}.  
  \end{equation*}
  Finally, we call \emph{arg-front of $\fs$}, or simply
  \emph{$\afront$ of $\fs$}, to any subset $\apf$ of the Pareto Front
  $P$ that is mapped by $\ofv$ into $\ifront_{\fs} \ofv$, in a
  one-to-one fashion. We will use the notation
  \begin{equation*}
    \apf = \afront_{\fs} \ofv.
  \end{equation*}
\end{definition}

The problem we want to solve is the following multi-objective
generalization of \<pbo> (Definition~\ref{def:pbo}).

\begin{definition}[\<moco>]
  \label{def:moco} Let $\ofv : \ds \rightarrow \os \subseteq \Nb^n$ be
  a \emph{multi-objective \<pb> function}, mapping the \emph{decision
    space $\ds \subseteq \PBool^n$} into the \emph{objective space
    $\os$}. Let $\fs \subseteq \ds$ be the feasible space of some
  \<pb> formula $\f$, with variables in $V$. Then,
  \begin{equation}
    \text{find } \afront_{\fs(\f)} \ofv.
  \end{equation}
  An instance will be denoted by the triple $\pbrac{\f, V,\ofv}$.
\end{definition}

Because the solutions of the problems are sets, bounds are now
\emph{bound sets} (Definition~\ref{def:bound-sets}). In the single objective
case, a bound is a value $l$ such that $\forall y = f (x): l \leq y $,
or equivalently, $\not\exists y = f (x): l > y$. This equivalence is
broken by the generalization. Each of the above defining properties of
a lower bound give rise to a differently flavoured comparison of sets
(Definitions \ref{def:coverage-set} and \ref{def:non-inferior-set}).

\begin{definition}[set coverage]
  \label{def:coverage-set}
  Let $A$ and $B$ be subsets of some decision space $\ds$, equipped
  with a multi-objective function $\ofv$. Then, \emph{$A$ covers $B$}
  iff every element of $B$ is dominated by some element of $A$, \ie,
  $\forall b \in B, \exists a \in A: a \<pd>[1] b,$ and
\emph{$A$ strictly covers $B$} iff
  $\forall b \in B, \exists a \in A: a \<pd>[0] b.$
\end{definition}

\begin{definition}[set non-inferiority]
  \label{def:non-inferior-set}
  Let $A$ and $B$ be subsets of some decision space $\ds$, equipped
  with a multi-objective function $\ofv$. Then \emph{$A$ is
    non-inferior to $B$} iff there is no element of $B$ that
  strictly-dominates an element of $A$,
  $\forall a \in A, b \in B: \lnot(a \<pd>[2] b),$ and \emph{$A$ is
    strictly non-inferior to $B$} iff
  $\forall a \in A, b \in B: \lnot(a \<pd>[3] b)$.
\end{definition}

The term {\em non-inferior} is admittedly a poor name choice. The only
favoring argument is that the stronger candidate {\em non-dominated} is
usually reserved for the fronts themselves.

Note that in the single objective case non-inferiority and coverage
are the same thing. Therefore, the next definition generalizes
correctly the notion of lower bound.

\begin{definition}[bound sets]\label{def:bound-sets}
  $L \subseteq \ds$ is a \emph{(strictly) lower bound set of
    $\fs \subseteq X$} iff $L$ (strictly) covers and is
  (strictly) non-inferior to $\fs$. If $L$ is a lower bound set of
  $\fs$, we say $L \<pd>[1] \fs$. If it is a strictly lower bound set,
  we say $L \<pd>[0] \fs$.
\end{definition}

One way to generate a lower bound set of some Pareto front is to solve
a related problem, where the formula is replaced by a relaxed version
(Definition~\ref{def:rcnf}).

In order to guide the search we will need to embed dominance relations
into \<cnf> formulas. For instance, we are interested in ensuring we
do not spend time looking for solutions that are dominated by some
other known feasible solution. But how can we express it on a \<sat>
query? We need to translate the requirement into a \<cnf> formula. A
particular example of such a translator is called an \emph{unary
  counter}. In particular, they have been used to implement efficient
\emph{\<pb> satisfiabity solvers} that simply forward \<pb> queries
into a \<sat> solver after translating them.

\begin{definition}[Unary Counter]\label{def:unary-counter}
  Let $f_{i}: \PBool^m \into \Nb$ be a \<pb> function and set $V$ be an
  ordered set of variables that parametrize the domain of $f_{i}$,
\begin{gather}
  \label{eq:39}
  V = \brac{x_1, \ldots, x_m}, f_{i}(\mx) = f_{i}(x_1, \ldots, x_m).
\end{gather}
Consider the \<cnf> formula $\fa$ with variables $V\cup O$, where
$O \cap V = \emptyset$ and $O$ contains one variable $\ov_{i,k}$ for each
value $k \in \Nb: \exists \mx: k = f_{i}(\mx)$.  The elements of $\ovs$
are the \emph{order variables}. We call the tuple
$\pbrac{f_{i}, V, \ovs, \fa}$ an unary counter of $f_{i}$ iff all feasible models $\m$
of $\fa$ satisfy
\begin{equation}
  \label{eq:counter_semantics}
  f_{i}(\mx) \geq k \implies \ov_{i,k},\quad \mx = \m(V).
\end{equation}
\end{definition}

\section{Unsatisfiability-based Algorithm}
\label{sec:unsat}
Although core-guided algorithms for single-objective problems such as Maximum
Satisfiability~\cite{DBLP:conf/vlsid/FuM07,wmsu3-corr07,DBLP:conf/sat/ManquinhoSP09,DBLP:conf/sat/AnsoteguiBL09,DBLP:conf/cp/AnsoteguiBGL13}
have been initially proposed more than one decade ago, to the best of our
knowledge, there is no such algorithm for MOCO. The main goal of our algorithm
is to take advantage of unsatisfiable cores identified by a SAT solver in order
to lazily expand the allowed search space.

\subsection{Algorithm Description}
\def\scale{.9}
\def\Alpha{\mathrm{A}}
\def\visible<#1>{}
      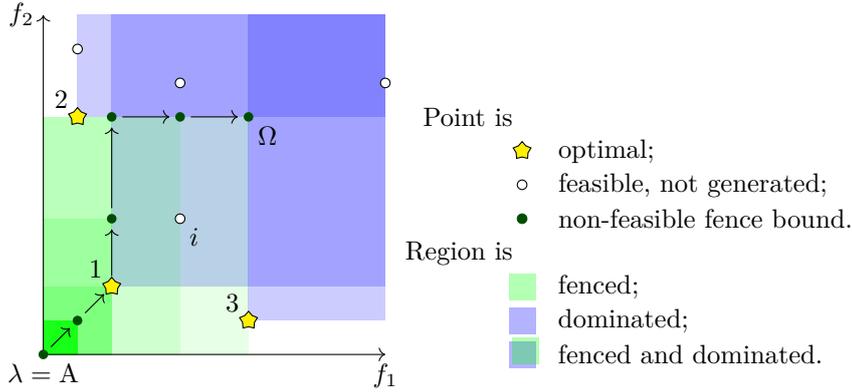
\begin{figure}[t]
        {
          \begin{tikzpicture}[scale=\scale,ul/.style={dot,green!30!black},
            myScale/.style={xscale=\maxx/10,yscale=\maxy/10},
            myArrow/.style={shorten >=4,shorten <=4,->}
            ]
            \minMax{}
            \axes{$f_1$}{$f_2$}
            \path  (min -| max) ++(0,2) to coordinate (legend) (max);
            \path (legend) ++(2,0) coordinate (legend) 
            node[anchor=east] {Point is};      
            \path (legend) ++(0,-.5) coordinate (legend) 
            node[myStar] {} node[anchor=west,xshift=10pt] {optimal;};
            \path (legend) ++(0,-.5) coordinate (legend) 
            node[anchor=west,xshift=10pt] {feasible, not generated;} [draw,dot] circle[];
            \path (legend) ++(0,-.5) coordinate (legend)
            node[anchor=west,xshift=10pt] {non-feasible fence bound.} [draw,fill,dot,ul] circle[];
            \path (legend) ++(0,-.5) coordinate (legend)
            node[anchor=east] {Region is};      
            \path (legend) ++(0,-.5) coordinate (legend) 
            node[anchor=west,xshift=10pt] {fenced;} node[mySwatch,green]{};
            \path (legend) ++(0,-.5) coordinate (legend)
            node[anchor=west,xshift=10pt] {dominated;} node[mySwatch,blue]{};
            \path (legend) ++(0,-.5) coordinate (legend)
            node[anchor=west,xshift=10pt] {fenced and dominated.}
            node[mySwatch,blue]{} ++(.05, .05)
            node[mySwatch,green]{};
            
            \pgfmathsetseed{10}
            \node[below right,white] at (min) {\pgfcurrentseed};
      \foreach \c/\l in {1/2,2/3,3/5,4/4,5/1,6/6,7/7,8/8,9/9,10/10}{
        \randx \randy
        \path
        (\xpoint, \ypoint) coordinate (p\l);
      }
      \foreach \c in {2,3,1}{
        \randx \randy
        \path
        (p\c)  node[above left] {\c};
      }
      \path[scaled] (p4) ++ (1,1) coordinate (p4)  node[below right] {$i$};

            \visible<2>{
              \path (0,0) coordinate (ul); 
              \node[anchor = north] at (ul) {$\lambda = \Alpha$};
              \fill[ul] (ul)  circle[];
              \dominating{ul}{green}{.1};
            }
            \draw[myScale, myArrow] (ul) -- (1,1) coordinate (ul);
            \fill[ul] (ul)  circle[];
            \dominating{ul}{green}{.8};

            \visible<4>{
              \draw[myScale, myArrow] (ul) -- (2,2) coordinate (ul);
              \fill[ul] (ul)  circle[];
              \dominating{ul}{green}{.2};
            }
            \visible<5->{
              \dominator{p1}{blue};
            }
            
            \visible<6->{\node[myStar] at  (p1) {} ;}
            \visible<7>{
              \draw[myScale, myArrow] (ul) -- (2,4) coordinate (ul);
              \fill[ul] (ul)  circle[];
              \visible<7>{\dominating{ul}{green}{.15};}
            }

            \visible<9>{
              \draw[myScale, myArrow] (ul) -- (2,7) coordinate (ul);
              \fill[ul] (ul)  circle[];
              \dominating{ul}{green}{.1};
            }
            \visible<10->{
              \dominator{p2}{blue};
            }
            \visible<10->{\node[myStar] at  (p2) {} ;}

            \visible<11>{
              \draw[myScale, myArrow] (ul) -- (ul -| p4)  coordinate (ul);
              \fill[ul] (ul)  circle[];
              \dominating{ul}{green}{.1};
            }

            \draw[myScale, myArrow] (ul) -- (6,7)  coordinate (ul);
            \fill[ul] (ul)  circle[];
            \dominating{ul}{green}{.1};
            \node[anchor = north west] at (ul) {$\Omega$};

            \visible<12->{
              \fill[dot] (p3)  circle[];
              \dominator{p3}{blue};
            }
            \visible<13->{\node[myStar] at  (p3) {} ;}
      \foreach \c in {4,7,8,9}{
        \draw[dot] (p\c)  circle[];
      }
          \end{tikzpicture}
        }
        \caption{Illustration of a run of \US*
          (Algorithm~\ref{alg:moco-unsat}) in the objective space.
          The img-front is the set $\brac{1,2,3}$. The fence bound
          $\lambda$ gets updated at each iteration of the while cycle
          at line~\ref{algl:main-loop}, starting at $\Alpha$ and
          ending at $\Omega$. The arrows are guided by the core $\k$
          (line~\ref{alg:19}). The green shading represents the
          evolution of the fence. Darker regions have been fenced for
          longer. The blue regions are blocked by optimal
          points. Darker regions are dominated by more points. We will
          be done in $7$ iterations. After verifying that $\Alpha$ is
          not feasible, we are instructed by the cores $k$ to move
          along the diagonal twice. We find point $1$
          fenced. Therefore the associated $\mx$ is copied into $\inc$
          and the dominated region is blocked. We extend $\lambda$
          twice, and find point $2$. After moving once more, we find
          part of the fence blocked, and the point branded with $i$ is
          never generated. The next movement stations $\lambda$ at
          $\Omega$. Point $3$ is found. The Oracle acknowledges we are
          done, by returning $\k =\emptyset$ (line~\ref{alg:9}): she
          knows that no movement of $\lambda$ will extend
          $\inc$. }\label{fig:unsatSat}
      \end{figure}

\begin{algorithm}[t]
\caption{Core-Guided MOCO solver \US*}
    \label{alg:moco-unsat}
        \SetKwFunction{EncodeOrder}{EncodeOrder}
        \SetKwFunction{myNext}{next}
        \mocoInput
        \mocoOutput
        \SetData{m}{|F|}\\
        \SetData{\inc}{\emptyset}\\
        \SetData{(\fa, \brac{\ovs_i}_{1 \leq i \leq m})}{\EncodeOrder(F, V)}\tcp*[f]{build unary counters for each $f_i$}\\\label{algl:cnf-enc}
        \SetData{\fa}{\fa \cup \f}\label{algl:init-work-formula}\\
        \SetData{\lambda}{\langle 0, 0, \ldots, 0 \rangle}\tcp*[f]{init fence upper bound}\label{alg:lambda}\\ 
        \While{true}{ \label{algl:main-loop}
            \SetData{\A}{\brac{\brac{\lnot o_{i, \myNext(i,\lambda)}} : i \in 1 \ldots m} }\tcp*[f]{enforce fence $\lambda$} \label{algl:assumptions}\\
            \SetData{(\m,\k)}{\satcall{\fa}{}{\A}}\\
            \While(\label{alg:15}){$\m\neq \emptyset$}{ \label{algl:start-inner-loop}
              \SetData{\mx}{\m(V)}\\\label{alg:14}
              \SetData{\inc}{\inc \setminus \brac{\mxx\in \inc: \mx\<pd>[1] \mxx} \cup \brac{\mx} }\\\label{alg:16}
              \SetData{\fa}{\fa \cup \brac{\bigvee\limits_{i = 1}^{m}{\lnot o_{i, f_i(\mx)}}}}\tcp*[f]{block region dominated by $\mx$}\\\label{algl:phi-updt}
              \SetData{(\m,\k)}{\satcall{\fa}{}{\A}}\\\label{alg:18}
 } 
{\eIf(\tcp*[f]{the fence was not the problem}\label{alg:9}){$\k = \emptyset$}{ \KwRet{$\inc$} \tcp*[f]{arg-front found}\label{algl:return}
            }{ \ForEach(\tcp*[f]{expand fence, as suggested by $\k$}){
                $\brac{\lnot o_{i, k}} \in \k$} {
                \SetData{\lambda_i}{k}\label{alg:19}
              } } } } \label{algl:main-loop-end}
\end{algorithm}


Algorithm~\ref{alg:moco-unsat} presents the pseudo-code for an exact
core-guided SAT algorithm for MOCO. Figure~\ref{fig:unsatSat}
illustrates an abstract execution of the algorithm.

Let $\pbrac{\f, V, \ofv}$ be a MOCO instance.  Recall that $\f$
denotes the set of \<pb> constraints, $V$ is the set of variables and
$\ofv$ denotes the list of $m$ objective functions.

First, the algorithm starts by building a working formula with the
problem constraints and an unary counter for each objective function
(lines~\ref{algl:cnf-enc}-\ref{algl:init-work-formula}).  Next, a
vector $\lambda$ of size $m$ is initialized with the lower bound of
each objective function (line~\ref{alg:lambda}), assumed to be $0$.

At each iteration of the main cycle the assumptions $\A$ are assembled
from order variables $\ov$, chosen with the value of $\lambda$ in mind
(line~\ref{algl:assumptions}). The call to
$\myNext(i,\lambda)$\footnote{May be replaced by $\lambda_i + 1$.}
returns the next smallest value belonging to the image of the
objective $i$.
Given the semantics of the order variables $\ov_{i,k}$
(Definition~\ref{def:unary-counter}), the tuple $\lambda$
\emph{fences} the search space, \ie, $\m$ satisfies $\a$ only if the
corresponding tuple $\mx$ satisfies $F(\mx) \<pd>[1]
\lambda$. If the \<sat> call (line~\ref{alg:14}) returns a solution
(\ie, $\m \neq \emptyset$),
$\mx$ is stored in and all dominated solutions are removed from
$I$ (line~\ref{alg:16}). Moreover, one can readily block all feasible
solutions dominated by
$\mx$ using a single clause (line~\ref{algl:phi-updt})\footnote{This
  technique had already been proposed when enumerating minimal
  models~\cite{DBLP:conf/cp/SohBTB17}.}.

Usually there are several feasible fenced solutions. This is so
because the algorithm may increase multiple entries of $\lambda$ at
once. In any case, the inner while loop
(lines~\ref{alg:15}-\ref{alg:18}) collects all such solutions.



When the working formula $\fa$ becomes unsatisfiable, the SAT solver
provides an unsatisfiable core $\k$. If $\k$ is empty, then the
unsatisfiability does not depend on the assumptions, \ie, it does not
depend on the bounds imposed on the objective functions.  At that
point, we can conclude that no more solutions exist that are both
satisfiable and not dominated by an element of $\inc$. As a result,
the algorithm can safely terminate (line~\ref{algl:return}).
Otherwise, the literals in $\k$ denote a subset of the fence walls
$\lambda_i$ that may be too restrictive, in the sense that unless we
increment them (line~\ref{alg:19}) no new non-dominated solutions can
be found.



\subsection{Algorithm Properties}

\begin{lemma}
  \label{prop:7}
  The \emph{img-front} $\ipf$ of $\inc \cup \fs(\fa)$(Definition~\ref{def:pareto-front}) is an invariant of the inner loop in lines~\ref{alg:15}-\ref{alg:18}.

\end{lemma}

\begin{proof}
  Consider some particular iteration of the internal
  loop. Line~\ref{alg:16} and line \ref{algl:phi-updt} remove all
  elements of $\inc \cup \fs(\fa)$ that are dominated by the feasible
  point $\mx$. Line~\ref{alg:16} filters the explicit set $\inc$,
  line~\ref{algl:phi-updt} filters the implicit set
  $\fs(\fa)$. Solutions that are strictly dominated by $\mx$ cannot be
  mapped into an element of $\ipf$. The other solutions $\mxx$ that
  are filtered out must attain the same objective vector attained by
  $\mx$, $\ofv(\mxx) = \ofv(\mx)$.  Because $\mx$ is also inserted at
  line~\ref{alg:16}, removing $\mxx$ will not disturb $\ipf$.
\end{proof}

\begin{lemma}\label{prop:8}
  At the start of each iteration of the external loop (lines~\ref{algl:main-loop}-\ref{algl:main-loop-end}),
  every solution in $\inc$ is optimal, and no two elements of $\inc$ attain the same objective vector.
\end{lemma}

\begin{proof}
  We prove this by contradiction. Assume that there is a non-optimal
  solution $\mx \in \inc$ at the start of the external loop (line~\ref{algl:main-loop}).
  In the first iteration, this does not occur because $\inc$ is empty.
  Hence, this can only occur if the inner loop (lines~\ref{alg:15}-\ref{alg:18})
  finishes with a non-optimal solution $\mx \in \inc$.
  
  The inner loop (lines~\ref{alg:15}-\ref{alg:18}) enumerates
  solutions inside the fence defined by $\lambda$. We know that
  $\ofv(\mx) \<pd>[1] \lambda$ because it is inside the fence and the
  values in $\lambda$ never decrease.  If $\mx$ is non-optimal, then
  there must be an optimal solution $\mxx$ such that
  $\ofv(\mxx) \<pd>[0] \ofv(\mx) (\<pd>[1] \lambda)$. Hence, $\mxx$ is
  also inside the fence. As a result, $\mxx$ must be found before the
  inner loop finishes, since at each iteration only dominated
  solutions are blocked (line~\ref{algl:phi-updt}).  If $\mx$ is found
  before $\mxx$, then $\mx$ is excluded from $\inc$
  (line~\ref{alg:16}) when $\mxx$ is found. Otherwise, if $\mxx$ is
  found first, then $\mx$ is not found by the SAT solver (blocked at
  line~\ref{algl:phi-updt}) because it is dominated by $\mxx$.
  Therefore, we cannot have a non-optimal solution $\mx \in \inc$ at
  the end of the inner loop or at the start of each iteration of the
  external loop (lines~\ref{algl:main-loop}-\ref{algl:main-loop-end}).
  



  Furthermore, no two elements of $\inc$ attain the same objective vector
  since when a solution $\mx$ is found, all other solutions $\mxx$ such
  that $\ofv(\mx)=\ofv(\mxx)$ are also blocked (line~\ref{algl:phi-updt}).
\end{proof}

Lemma~\ref{prop:8} establishes a weaker form of \emph{anytime
  optimality}. The elements of the incumbent list $\inc$ are not
necessarily optimal at anytime, but they are optimal immediately after
completing the inner loop. It is easy enough to make it anytime
optimal. This could be achieved if the algorithm refrains from adding
solutions directly to $\inc$ in the inner loop and maintain a
secondary list, where it stores the solutions that are still not
necessarily optimal. This list takes the role of $\inc$ inside the
inner loop. After completing the inner loop, all elements of the
secondary list are optimal, and can be safely transferred to the main
list $I$.

\begin{lemma}
  Algorithm~\ref{alg:moco-unsat} is sound. 
\end{lemma}

\begin{proof}
  If the algorithm returns, $\fs(\fa\land\A) = \emptyset$. Because
  $\k$ is empty, no core of the unsatisfiable formula $\fa\land\A$
  intersects $\A$, and $\fa$ is also unsatisfiable,
  $\fs(\fa) = \emptyset$. Using Lemma~\ref{prop:7} both at the end and
  at the start of the course of the algorithm, the img-front of $I$ is
  the img-front of $\fs(\fa)$, with $\fa$ given by
  line~\ref{algl:init-work-formula}. Because the order variables are
  only restricted by the unitary counter formula, the img-front of
  $\fs(\fa)$ is the img-front of $\fs(\f)$. Therefore $\inc$ must
  contain an arg-front of the problem. Using Lemma~\ref{prop:8}, every
  element of $\inc$ is optimal, and there is no pair
  $\mx,\mxx\in \inc$ such that $\ofv(\mx)=\ofv(\mxx)$. Therefore,
  $\inc$ is an arg-front of the \moco++ instance.
\end{proof}

\begin{lemma}
  Algorithm~\ref{alg:moco-unsat} is complete.
\end{lemma}

\begin{proof}
  The inner loop will always come to fruition, because in the worst
  case it will generate every feasible solution dominated by the
  current $\lambda$ once, and the feasible space is finite.

  If the algorithm does not return for some particular instance, then
  $\k$ is never empty. In that case, every iteration of the external
  loop starting at line~\ref{algl:main-loop} will increase at least
  one of the entries of $\lambda$. Eventually, one entry $i$ must
  achieve the upper limit of $f_i$, and the order variable retrieved
  by $\ov_{i,\lambda_i + 1}$ will not exist. Because the evolution of
  $\lambda_i$ is monotonous, the assumptions will contain at most
  $m - 1$ variables, from that point on. By the same token, the
  assumptions $\A$ will eventually be empty, and so must be
  $\k\ \subseteq \A$, contradicting the assumption that the algorithm
  never terminates.
\end{proof}

\section{Hitting Set-based Algorithm}
\label{sec:hitting-set}
In this section we propose a \gls{moco} solver based on the enumeration
of hitting sets. 
We briefly motivate the algorithm, and prove its correctness and soundness.

The main idea is to compute a sequence of relaxations $\rf$ of the
formula $\f$, and solve the corresponding problems. The front $\rapf$
of the relaxed problem gets incrementally closer to the desired front
$\apf$, and will eventually reach it in a finite amount of time.

\subsection{Algorithm Description}

\begin{algorithm}[t]
  \caption{Hitting-Sets based MOCO solver \HS*}
  \label{alg:hitting-sets-x}
  \mocoInput
  \mocoOutput
  \SetData{\rf}{\emptyset}\tcp*[f]{relaxed formula $\rf$ is initially empty}\\\label{alg:3}
  \While(\label{alg:2}){true}{
    \SetData{\diagno}{\emptyset}\\\label{alg:1}
    \SetData{\rapf}{\afront_\rf \ofv}\tcp*[f]{use auxiliar solver}\\\label{alg:28}
    \ForEach(\tcp*[f]{diagnose $\f$-infeasibility of
      $\rapf$}\label{alg:6}){$\mx \in \rapf$}{
      \SetData{\A_{\mx}}{\brac{\brac{l}, l \in \m: \m(V) = \mx}}\\\label{alg:10}
      \SetData{(\bullet, \k)}{\satcall{\f}{}{\A_{\mx}}}\\\label{alg:30}
      \If{$\k \neq \emptyset$}{
        \SetData{\diagno}{\k \cup \diagno}\\\label{alg:8}
      }
    }
    \If(\tcp*[f]{if $\rapf$ is $\f$-feasible}\label{alg:7}){$\diagno = \emptyset$}{
      \KwRet{$\rapf$}\tcp*[f]{arg-front found}\\}\label{alg:4}
    \ForEach{$\k \in \diagno$}{
      \SetData{\rf}{\rf \cup \brac{\lnot l, \brac{l} \in
          \k}}\tcp*[f]{Increment $\rf$, adding $\lnot \k$}\\\label{alg:5}
    }
  }
\end{algorithm}

\def\scale{.9}
\def\visible<#1>{}
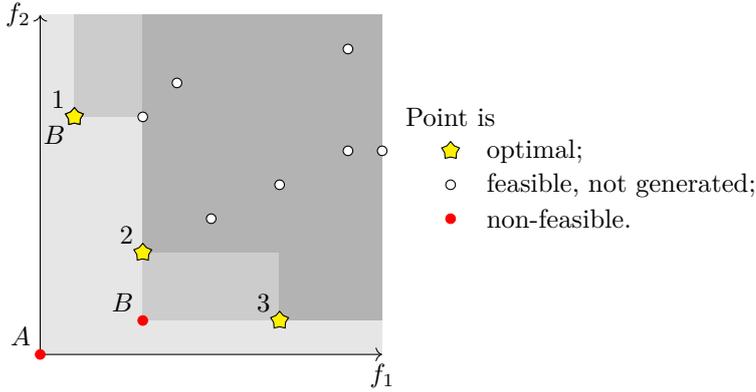
\begin{figure}[t]
  {
    \begin{tikzpicture}[scale=\scale]
      \minMax{}
      \axes{$f_1$}{$f_2$}
      \pgfmathsetseed{16}
      \node[below right,white] at (min) {\pgfcurrentseed};
      \path  (min -| max) ++(0,2) to coordinate (legend) (max);
      \path (legend) ++(1,0) coordinate (legend);
      \node[] at (legend) {Point is};      
      \path (legend) ++(0,-.5) coordinate (legend);
      \node[myStar] at (legend) {};
      \node[anchor=west,xshift=10pt] at (legend) {optimal;};
      \path (legend) ++(0,-.5) coordinate (legend);
      \node[anchor=west,xshift=10pt] at (legend) {feasible, not generated;};
      \draw[dot] (legend) circle[];
      \path (legend) ++(0,-.5) coordinate (legend);
      \node[anchor=west,xshift=10pt] at (legend) {non-feasible.};
      \draw[fill,dot,red] (legend) circle[];
      \foreach \c in {9,8,4,5,6,7,10}{
        \randx \randy
        \path
        (\xpoint, \ypoint)
        coordinate (p\c) node[above left] {};
        \draw[dot] (p\c)  circle[];
      }
      \foreach \c in {3,1,2}{
        \randx \randy
        \path
        (\xpoint, \ypoint)
        coordinate (p\c) node[above left] {\c};
        \draw[dot] (p\c)  circle[];
      }
      \path[scaled] (min) coordinate (lb1);
      \node[anchor=south east] at (lb1) {$A$};
      \visible<2-3>{\dominatorOpaque{lb1}{black!10!white};}
      \visible<3->{\draw[fill,dot,red] (lb1) circle[] ;}

      \visible<6->{ \path[scaled] (3,1) coordinate (lb1) (p1) coordinate (lb2);}
      \visible<6-8>{\dominatorOpaque{lb1}{{black!20!white}};}
      \visible<6>{\dominatorOpaque{lb2}{{black!20!white}};}
      \visible<7->{ 
        \node[myStar] at  (p1) {} ;}
      \visible<8->{\draw[fill,dot,red] (lb1) circle[] ;}      
      \node[myStar] at  (lb2) {};
      \node[anchor=south east] at (lb1) {$B$};
      \node[anchor=north east] at (lb2) {$B$};
      \visible<9->{ \path[scaled] (p2) coordinate (lb1) (p3) coordinate (lb3);}
      \visible<9>{\dominatorOpaque{lb1}{black!30!white};}
      \visible<9>{\dominatorOpaque{lb3}{black!30!white};}

      \visible<10->{ 
        \node[myStar] at  (lb1) {};
        \node[myStar] at  (lb3) {};}
    \end{tikzpicture}
  }
  \caption{ Illustration of a run of the
    \HS* (Algorithm~\ref{alg:hitting-sets-x}) in the
    objective space. The Pareto front is the set $\brac{1,2,3}$, and
    the feasible solutions are marked by \tikz{\draw[dot]
      circle[];}. For each iteration of the main while cycle at
    line~\ref{alg:2} we get a narrower lower bound
    $\rapf$ (line~\ref{alg:28}), culminating in the solution.  We are
    done in 3 iterations, marked by $A$, $B$ and \tikz{\node[myStar]
      {};}. The shading represents the number of iterations whose
    freshly found points dominate the region. The lighter tone was
    painted by $A$, the darker one by all three.  We start with the
    empty formula (line~\ref{alg:3}), and we get $A$. Because the only
    point in $A$ is not feasible, we tighten the
    relaxation (line~\ref{alg:5}). Iteration $B$ generates one feasible
    point, $1$, which is therefore optimal. Note that the region
    dominated by $1$ can be pruned from now on. The other point is used to
    tighten the formula once more. Lastly, the lower bound contains
    the feasible points $2$ and $3$ in addition to $1$ which was
    already found, and the algorithm stops.}\label{fig:hittingSets}
\end{figure}

Algorithm~\ref{alg:hitting-sets-x} contains the pseudo-code for our
hitting set-based algorithm for MOCO. Figure~\ref{fig:hittingSets}
illustrates an abstract execution of the algorithm.

The algorithm starts by setting the \emph{relaxed formula} $\rf$ to
the empty one (line~\ref{alg:3}).
The main cycle that starts at line~\ref{alg:2} will hone the relaxation
until we get the desired result.
At each iteration, we solve the current relaxed formula $\rf$ at
line~\ref{alg:28}. This will be accomplished by using some \<moco> solver.
Because this amounts to computing a lower bound set, the \US* algorithm,
previously described, is a good choice for the task.
We anticipate that it performs well for problems whose front is
in the vicinity of the origin, given that by construction the focus of
its search is biased to that region.  Notice that the first
relaxation's arg-front is the set that contains the origin only
(assuming all literals in the objective functions are positive). We
expect that the first few relaxations will stay close to it.

Next, for each element $\mx$ in $\rapf$ (the Pareto-front of $\rf$),
we check the $\f$-feasibility of $\m: \m(V) = \mx$, using the
assumptions mechanism, and returns a (possibly empty) core of
assumptions $\k$.  The assumptions $\A_{\mx}$ built at
line~\ref{alg:10} are a set of unitary clauses whose polarity is
inherited from $\m$,
\begin{equation} \label{eq:28}
  \begin{gathered}
    \m(x_i)  \implies x_i \in \A, \quad \lnot \m(x_i)  \implies \lnot x_i \in \A.
  \end{gathered}
\end{equation}

Given the interface of our \<sat> solver
(Definition~\ref{def:sat-solver}), the returned core $\k$ will be void
iff $\A_{\mx} \land \f$ is satisfiable.  In this case, $\mx$
corresponds to an optimal solution.

The \emph{diagnosis} $\diagno$ is central for the algorithm.
Intuitively, it reports if and why the relaxed problem's solution is
different from the true Pareto solution.
We add every non-empty $\k$ to the diagnosis $\diagno$ (line~\ref{alg:8}).
In the end, $\diagno$ is empty iff every element of the relaxed front $\rapf$
is $\f$-feasible. At that point, we have found a $\f$-feasible lower
bound set. All such sets are arg-fronts, and so the algorithm terminates
at line~\ref{alg:4}.
Otherwise, if $\diagno$ is not empty, then the found cores are
added to the relaxed formula $\rf$ (line~\ref{alg:5}).
This step ensures all tentative solutions found in line~\ref{alg:28}
hit the unsatisfiable cores found and that the algorithm advances in a
monotonous fashion towards the solution.
This will be further discussed in Lemma~\ref{prop:11}.

\subsection{Algorithm Properties}

Given a MOCO instance $\pbrac{\f,V,\ofv}$, the formula $\f$ encodes
the feasible space $\fs$ implicitly, which in turn defines the desired
front $\apf$. This is a many to one correspondence, in the sense that there
are many different values of $\ff$ that encode the same Pareto
front. It may happen that some of the counterpart instances are easier
to solve than the original one, which begs the question: given $\f$,
can we effectively find a simpler formula $\ff$ with the same Pareto
front?  This is the motto of the proposed algorithm. It is done by
iteratively honing a \emph{relaxed formula}
(Definition~\ref{def:rcnf}).

The main idea is to compute a sequence of relaxations that get
incrementally tighter. In that case, the corresponding front $\rapf$
gets incrementally closer to the desired front $\apf$,
\begin{align}
  \f && \implies && \rffn && \implies && \ldots  && \implies && \rffi,\\
  \apf &&  \<pd>[3] && \rapffn&&  \<pd>[3] && \ldots&&  \<pd>[3] && \rapffi,\label{eq:chain-lower-bounds}
\end{align}
where $\apf$ is one of the desired arg-fronts, and $\rapffx$ is an
arg-front of $\rffx$. 

\begin{lemma}
  \label{prop:super-feasible-space}
  Consider some multi-objective function $\ofv : \ds \rightarrow
  \os$. Let $\fs, \fss$ be subsets of $\ds$, such that
  $\fss \subseteq \fs$. Then, any arg-front of $\fss$ is a lower bound
  set of any arg-front of $\fs$ (Definition~\ref{def:bound-sets}),
  \ie,
    $\fss \subseteq \fs \implies \<apf>[\fss] \<pd>[1] \<apf>[\fs]$.
\end{lemma}

Lemma~\ref{prop:super-feasible-space} is true because optimizing
over a superset of some feasible space always returns a (non-strict)
lower bound set. In a sense, the optimization can only be more extreme
when applied to the superset. In particular, the feasible space of a
relaxed formula is a superset of the original one. This is why the
chain of $\<pd>[1]$ relations in Equation~\eqref{eq:chain-lower-bounds}
is correct.

\begin{lemma}\label{prop:2}
  Let $\f$ be a formula, $\fs \subseteq \ds$ be its feasible space and
  $\ofv : \ds \rightarrow \os$ be some multi-objective function. Let
   $L$ be a lower bound set of the Pareto front of $\fs$. Then, any
   element $x \in L$ that is feasible belongs to the Pareto front,
     $L \cap \fs \subseteq \pf$.
   If all elements $\mx \in L$ are feasible, then $L$ is an arg-front.
\end{lemma}

Lemma~\ref{prop:2} implies that every lower bound set with only
feasible elements must be itself an arg-front (this is an exact
analogy with the single-objective case, where {\em lower bound set} is
replaced by {\em infimum} and {\em arg-front} by {\em arg-min}.)
By construction of the diagnosis $\diagno$, this is equivalent to the
condition used in Algorithm~\ref{alg:hitting-sets-x} to decide if
it can terminate.

To ensure the sequence gets to $\apf$ in a finite number of steps, we
need more than a string of relaxations. Each entry $\rff$ must be
strictly tighter than the predecessor $\rf$.

\begin{lemma}\label{prop:11}
  Consider Algorithm~\ref{alg:hitting-sets-x}. Let $\rf$ be the
  relaxed formula at some iteration, and $\rff$ be the relaxed formula
  at the next iteration. Then, 
\begin{enumerate}
\item  $\rf$ relaxes $\rff$, \ie, $\rff \ent \rf$;
\item Both $\rf$ and $\rff$ relax $\f$, \ie, $\f \ent \rf, \f \ent \rff$;
\item $\rff$ does not relax $\rf$, \ie, $\rf \not\ent \rff$;
\end{enumerate}
 \end{lemma}
 \begin{proof}
   Each statement will be proven in turn.
   
   The first is true because $\rf \subseteq \rff$, by construction
   (line~\ref{alg:5}).

   We prove the second by induction on the number of
   iterations. Initially $\rf$ is empty.  Therefore, $\rf$ relaxes any
   formula, in particular $\f$. Assume $\f \ent \rf$ for some
   iteration. Consider one of the clauses $\lnot\k$ added at
   line~\ref{alg:5}. We know that $\f \land \k$ is unsatisfiable.
   Therefore,
   $ \f \land \k \ent \false \implies \lnot (\f \land \k) \ent \true
   \iff \f \ent \lnot \k$. Given the assumption $\f \ent \rf$, we get
   $\f \ent \rf \land \lnot \k$.  Repeating the process for the other
   added clauses $\lnot \k_i$, we get
   $\f \ent \rf \land \lnot \k_1 \ldots \land \lnot \k_n \equiv \rff$.

   Assume $\rff$ is a relaxation of $\rf$. Then, any $\rf$-feasible
   model $\m$ is also $\rff$-feasible. We will prove there is at least
   one model that violates this. To start, note that it only makes
   sense to consider $\rff$ if there is some non-empty core $\k$ in
   the diagnosis $\diagno$, otherwise the algorithm would have
   terminated before updating $\rf$ into $\rff$. Let $\k$ be one
   element of $\diagno$, generated at line~\ref{alg:30} while $\rf$ is
   current.  Consider the Boolean tuple $\mx \in \rapf$ used to build
   the assumptions of the query that generated $\k$. Let
   $\m: \m(V) = \mx$. The model $\m$ is $\rf$-feasible, because it is
   part of the arg-front of $\rf$. The model $\m$ satisfies $\k$
   because $\k \subseteq \A_{\mx}$ and the way $\A_{\mx}$ is
   constructed (line~\ref{alg:10}, Equation~\eqref{eq:28}). Therefore,
   $\m$ does not satisfy $\lnot \k$. Because
   $\lnot \k \subseteq \rff$, $\m$ cannot satisfy $\rff$, \ie, there
   is at least one $\rf$-feasible model that is not $\rff$-feasible.
 \end{proof}

 \begin{proposition}
  Algorithm~\ref{alg:hitting-sets-x} is sound.
\end{proposition}
\begin{proof}
  By Lemma~\ref{prop:11}, $\rf$ relaxes $\f$ and therefore $\rapf$
  solves a relaxation of the original problem. By
  Lemma~\ref{prop:super-feasible-space}, it is a lower bound set of
  $\apf$. When the algorithm returns, all elements of $\rapf$ are
  feasible. By Lemma~\ref{prop:2}, $\rapf$ must be an arg-front.
\end{proof}

\begin{proposition}
  Algorithm~\ref{alg:hitting-sets-x} is complete.
\end{proposition}

\begin{proof}
  Assume Algorithm~\ref{alg:hitting-sets-x} never ends, implying
  $\rapf$ is never completely feasible (\ie, $\rapf \nsubseteq
  \fs$). The number of relaxed feasible spaces $\rfs$ is finite. If
  Algorithm~\ref{alg:hitting-sets-x} does not end, it will enumerate
  all of them, never repeating any: at any iteration, the updated
  relaxed formula effectively blocks the reappearance of any feasible
  space seen before, because by Lemma~\ref{prop:11} the updated value
  $\rff$ strictly tightens $\rf$. Then, this sequence is necessarily
  finite, and so must be the number of iterations. But in that case,
  Algorithm~\ref{alg:hitting-sets-x} must end, and we have a
  contradiction.
\end{proof}

Consider the sequence whose entries
are the value of $\ofv(\rapf)$ computed at the beginning of each
iteration of the main cycle at line~\ref{alg:2}. The last
element of this sequence is the solution. It may happen that for some
$i$, the entries indexed by $i$ and $i+1$ are the same.  Therefore,
the sequence may include blocks of contiguous entries that share the
same value. In the worst case scenario, there are many different
arg-fronts for the same img-front, and the algorithm ends up
enumerating all of them without any movement in the objective space.
We expect the algorithm will be effective whenever a few of the
relaxed problems are enough to get to the full solution. Otherwise, we
can end up solving an exponential number of problems.

\section{Experimental Results}
\label{sec:results}
\def\scale{.9}

This section evaluates the performance of the algorithms proposed in
Sections~\ref{sec:unsat} and~\ref{sec:hitting-set}. These new algorithms
are compared against other SAT-based solvers for MOCO.

\subsection{Algorithms and Implementation}

The \CoreGuidedAlg algorithm proposed in Algorithm~\ref{alg:moco-unsat} uses the selection
delimiter encoding~\cite{DBLP:journals/constraints/KarpinskiP19} that has been
shown to be more compact. Next, the selection delimiter encoding is extended
such that an unary encoding is produced for each objective function.
Additionally, an order encoding~\cite{tamura2008sugar} is also used on the
unary representation of each objective function.
We refer the interested reader to the literature for further details on this
and other encodings~\cite{DBLP:series/faia/RousselM09,DBLP:conf/cp/0001MM15,DBLP:journals/constraints/KarpinskiP19,DBLP:conf/sat/KarpinskiP20}. Observe that any unary encoding from PB into CNF
can be used.

The \CoreGuidedAlgStrat algorithm is an alternative implementation of
Algorithm~\ref{alg:moco-unsat} that uses the well-known technique of
stratification~\cite{DBLP:conf/cp/AnsoteguiBGL12,DBLP:conf/ijcai/MenciaPM15,DBLP:conf/ijcai/Terra-NevesLM18}.
In this case, literals in each objective function are split into
partitions according to their weights. \CoreGuidedAlgStrat
first solves the MOCO instance considering only the literals in the
partition with highest coefficients. Next, after finding an
approximation of the Pareto front using only part of the objective
functions, the next partition with highest coefficients for each
objective function is uncovered. This process is repeated until all
literals in all objective functions are represented.

The \HittingSetsAlg algorithm implements Algorithm~\ref{alg:hitting-sets-x}.
This hitting set based approach uses Algorithm~\ref{alg:moco-unsat}
in order to find the relaxed arg-front (line~\ref{alg:28} of
Algorithm~\ref{alg:hitting-sets-x}).

The \PMinimalAlg algorithm implements a SAT-UNSAT based approach based on the
enumeration of $P$-Minimal models~\cite{DBLP:conf/cp/SohBTB17}.
This algorithm is implemented with the same PB to CNF encoding as the \CoreGuidedAlg.
Finally, the \ParetoMCSAlg is based on the stratified enumeration of
Minimal Correction Subsets. We used the publicly available implementation
of \ParetoMCSAlg\footnote{\url{https://gitlab.ow2.org/sat4j/moco}}.

\subsection{Benchmark Sets}

The following MOCO problems were considered:
the multi-objective Development Assurance Level (DAL) Problem~\cite{DBLP:conf/safecomp/BieberDS11},
the multi-objective Flying Tourist Problem (FTP)~\cite{DBLP:journals/eswa/MarquesRR19}, and
the multi-objective Set Covering (SC) Problem~\cite{DBLP:conf/cp/BergmanC16,DBLP:conf/cp/SohBTB17}.

The DAL benchmark set encodes different levels of rigor of the development of a software or
hardware component of an aircraft. The development assurance level defines the
assurance activities aimed at eliminating design and coding errors that could affect the
safety of an aircraft. The goal is to allocate the smallest DAL to functions in order to
decrease the development costs~\footnote{The benchmark instances are available online at
\url{https://www.lifl.fr/LION9/challenge.php}. Although they define a lexicographic order
to the objective functions, in the context of this paper, we ignore this order and compute
the Pareto front.}.

The FTP benchmark instances encode the problem of a tourist that is
looking for a flight travel route to visit $n$ cities. The tourist
defines its home city, the start and end of the route. She specifies
the number of days $d_i$ to be spent on each city $c_i$
($1 \le i \le n$) and also a time window for the complete trip. The
problem is to find the route that minimizes the time spent on
flights and the sum of the prices of the tickets.

The SC benchmark is a generalization of the classical set covering
problem. Let $X$ be some ground set and $A$ a cover of $X$. Each
element in $A$ has an associated cost tuple. The goal is to find a
cover of $X$ contained in $A$ that Pareto-optimizes the overall cost.

\subsection{Experimental Setup and Evaluation Metrics}

All results were obtained on a Intel Xeon Silver 4110 CPU @ 2.10GHz,
with 64~GB of RAM. Each tool was executed on each instance with a time out
of 1 hour and 10~GB of memory.

Finding the whole Pareto front is extremely hard for most problem
instances. All tested algorithms are exact, but in most cases only an
approximation of the Pareto front could be found within the time limit.
In order to evaluate the quality of the approximations provided by each
tool, we use the Hypervolume (HV)~\cite{DBLP:phd/dnb/Zitzler99} indicator.
Since the HV indicator needs a reference front, for each benchmark
instance we combined the approximations produced by all algorithms to
build the reference front.
HV is a metric that measures the volume of the objective space dominated
by an approximation, up to a given reference point, and thus larger
values are preferred. A normalization procedure is carried out so that
the values of HV are always between 0 and 1.

\subsection{Results and Analysis}

Figure~\ref{fig:DALBOBHV} shows a cactus plot of all tools on the DAL
benchmark set. In this case, the SAT-UNSAT approach from \PMinimalAlg
is the overall best performing algorithm. The size of the encoding of
the objective functions is not very large. The \CoreGuidedAlgStrat and
\HittingSetsAlg algorithms take more iterations to find the set
of solutions in the Pareto front. This is more notorious in \HittingSetsAlg,
as many hitting sets have to be enumerated.
Nevertheless, in some instances, the number of MCS is large and
the \ParetoMCSAlg was unable to find MCSs in the Pareto front.
Overall, its performance is similar to our algorithms.

\begin{figure}[t]
  \centering
  \includegraphics[scale=\scale]{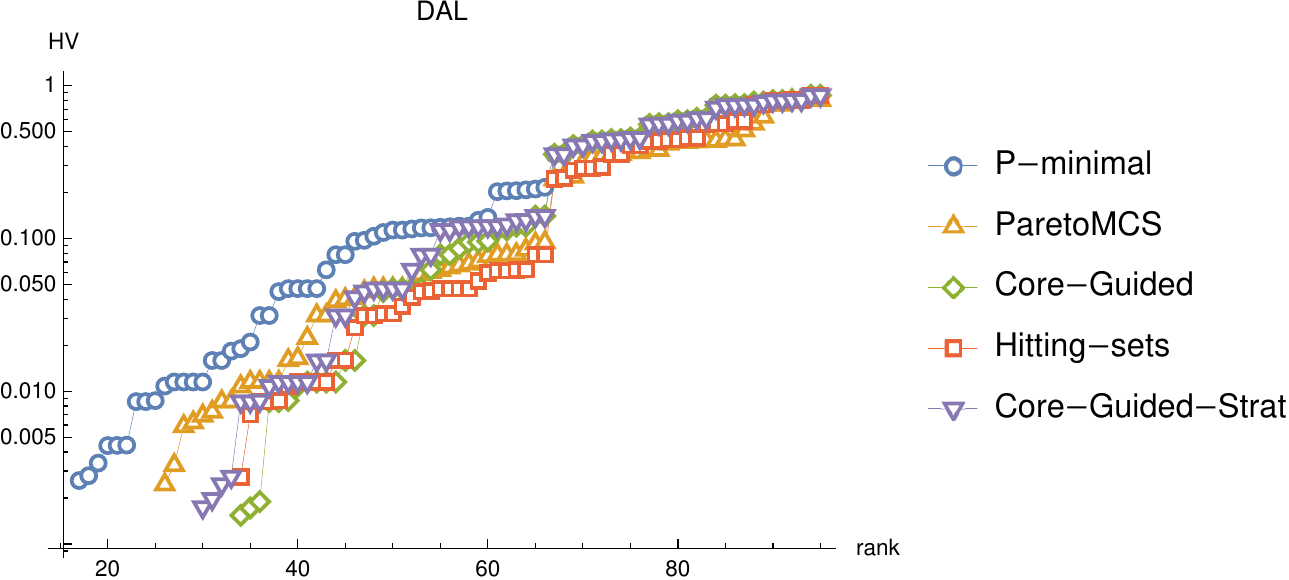}  
  \caption{Comparison of the $HV$ results for the DAL instances. Each
    series is sorted independently, smaller values first. Vertical
    scale is logarithmical.}
  \label{fig:DALBOBHV}
\end{figure}

The results for the FTP benchmark set are provided in
Figure~\ref{fig:FTPBOBHV}. In this case, the use of stratification
shows to be crucial. Observe that both \CoreGuidedAlgStrat and
\ParetoMCSAlg use this technique.
Note that \ParetoMCSAlg does not have an explicit representation
of the objective functions. However, the representation used in
\CoreGuidedAlgStrat is still effective for these instances, despite
some large coefficients in the objective functions.
As a result, these algorithms are able to quickly find a solution close
to the Pareto front, providing better overall performance than
the other algorithms.

\begin{figure}[t]
  \centering
  \includegraphics[scale=\scale]{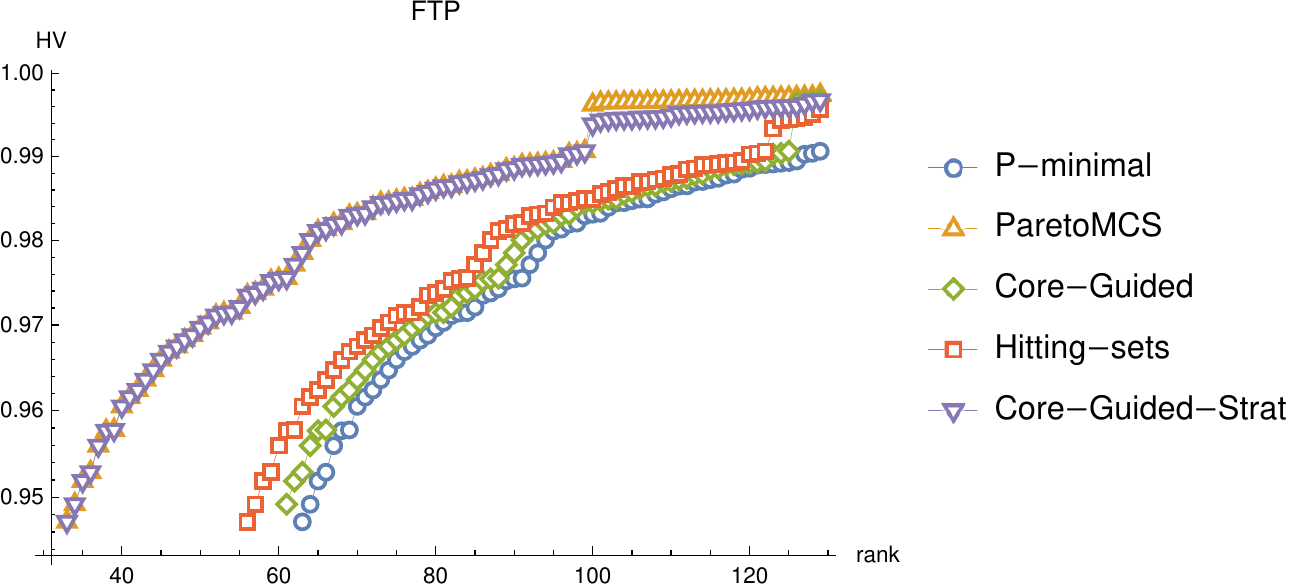}  
  \caption{Comparison of the $HV$ results for the FTP instances. Vertical scale is logarithmical.
  }
\label{fig:FTPBOBHV}
\end{figure}

Figure~\ref{fig:SCBOBHV} shows the results for the SC benchmark set.
In these instances, the \CoreGuidedAlgStrat was able to outperform
all other algorithms, since it does not need to relax all variables
to find solutions in the Pareto front. Note that when \HittingSetsAlg
is able to find solutions, these are in the Pareto front, providing
higher HV values in these cases. However, a common feature of the
\HittingSetsAlg is the need to enumerate many hitting sets before
being able to find a feasible solution.
Despite using stratification, \ParetoMCSAlg is unable to perform
as well as \CoreGuidedAlgStrat.

\begin{figure}[t]
  \centering
  \includegraphics[scale=\scale]{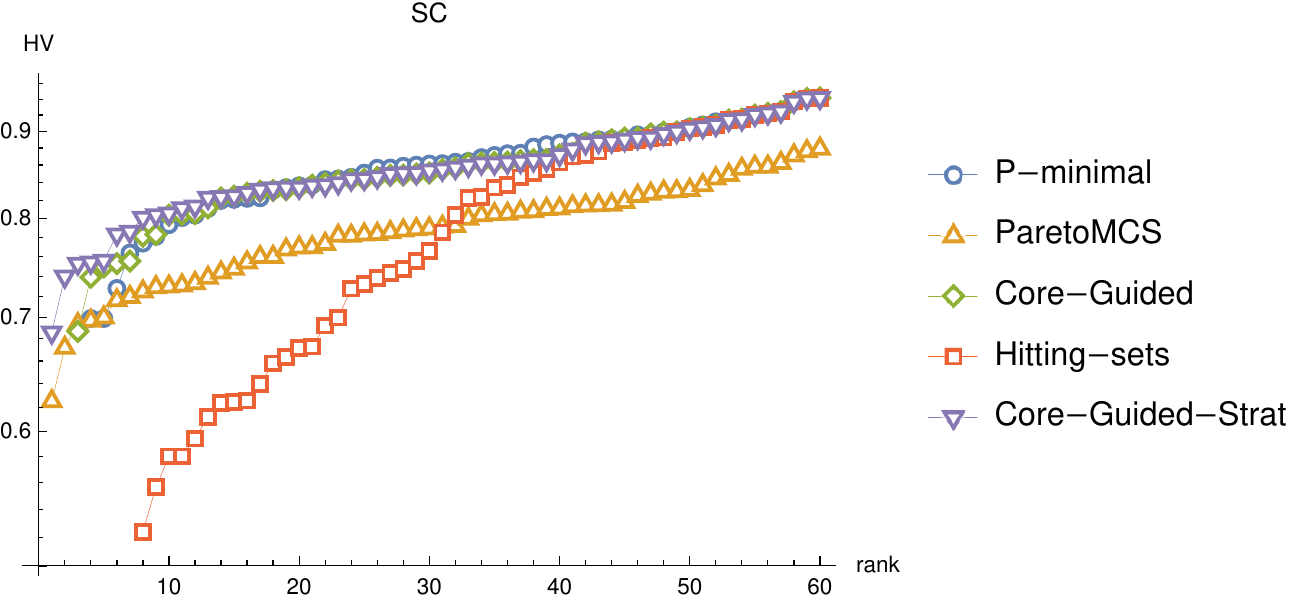}  
  \caption{Comparison of the $HV$ results for the SC
    instances. Vertical scale is logarithmical.}
\label{fig:SCBOBHV}
\end{figure}

\section{Conclusions}
\label{sec:conc}

This paper proposes two new algorithms for Multi-Objective Combinatorial
Optimization (MOCO). The first is a core-guided approach, while the
second is based on the enumeration of hitting sets.
These are the first algorithms for MOCO that follow
an UNSAT-SAT strategy.

Experimental results on three different sets of benchmark instances
show that the new core-guided approach results in a robust algorithm
that is competitive or outperforms other SAT-based algorithms for
MOCO. Using unary counters to express Pareto dominance in CNF shows to
be an effective way to harness the power of SAT solvers in solving
MOCO. The ability to express concepts related to dominance makes the
algorithms conceptually simple, and is therefore a useful tool in
developing other MOCO solvers based on SAT Oracles.

The algorithm based on the hitting set approach uses the
core-guided MOCO algorithm to incrementally enumerate the
hitting sets. The current implementation of the hitting set
algorithm is not competitive in some benchmark
instances due to having to enumerate a large number of
hitting sets.
The usage of the well-known stratification technique has a
significant impact on the performance of the core-guided algorithm
in several sets of instances. This technique is yet to be
implemented in the hitting set approach and we expect
that including stratification will allow to cut a large number
of the algorithm's iterations.

Overall, the new algorithms extend the state of the art
algorithms for MOCO based on calls to a SAT solver by
complementing the existing tools.

\bibliography{references}
\printacronyms
\printglossary[style=symbols]
\end{document}